\documentclass{article}

\usepackage[preprint,nonatbib]{neurips_2021}

\usepackage[utf8]{inputenc} % allow utf-8 input
\usepackage[T1]{fontenc}    % use 8-bit T1 fonts
\usepackage{hyperref}       % hyperlinks
\usepackage{url}            % simple URL typesetting
\usepackage{booktabs}       % professional-quality tables
\usepackage{amsfonts}       % blackboard math symbols
\usepackage{nicefrac}       % compact symbols for 1/2, etc.
\usepackage{microtype}      % microtypography
\usepackage{xcolor}         % colors

\usepackage{amsthm}
\usepackage{xspace}
\usepackage{commath}
\usepackage{float}
\usepackage[linesnumbered, boxed, ruled]{algorithm2e}
\usepackage{csquotes}
\usepackage{mathtools}
\usepackage{tablefootnote}

\newtheorem{theorem}{Theorem}[section]
\newtheorem{corollary}[theorem]{Corollary}
\newtheorem{lemma}[theorem]{Lemma}

\newtheorem{definition}{Definition}

\newcommand{\eps}{\varepsilon}
\let\oldnl\nl% Store \nl in \oldnl
\newcommand{\nonl}{\renewcommand{\nl}{\let\nl\oldnl}}% Remove line number for one line
\newcommand{\expect}[2]{\mathop{\mathbb{E}}_{#1} \left[ #2 \right]}
\newcommand\ceil[1]{\left\lceil #1 \right\rceil}

\newcommand{\naturals}{\mathbb{N}}
\newcommand{\reals}{\mathbb{R}}
\DeclareMathOperator*{\argmax}{arg\,max}

\renewcommand{\AE}{\mbox{AE}\xspace}
\newcommand{\UCB}{\mbox{UCB}\xspace}
\newcommand{\SDPAE}{\mbox{SDP-AE}\xspace}
\newcommand{\VBSDPAE}{\mbox{VB-SDP-AE}\xspace}

\allowdisplaybreaks

\title{Differentially Private Multi-Armed Bandits in the Shuffle Model}

\author{
  Jay Tenenbaum\thanks{Google Research. \href{mailto:jayten@google.com}{jayten@google.com}.}
  \And
  Haim Kaplan\thanks{Blavatnik School of Computer Science, Tel Aviv University and Google Research. \href{mailto:haimk@tau.ac.il}{haimk@tau.ac.il}. 
}
  \And 
  Yishay Mansour\thanks{Blavatnik School of Computer Science, Tel Aviv University and Google Research. \href{mailto:mansour.yishay@gmail.com}{mansour.yishay@gmail.com}. 
}
  \And 
  Uri Stemmer\thanks{Blavatnik School of Computer Science, Tel Aviv University and Google Research. \href{mailto:u@uri.co.il}{u@uri.co.il}.
}
}

\begin{document}
\maketitle

\begin{abstract}
    We give an $(\eps,\delta)$-differentially private algorithm for the multi-armed bandit (MAB) problem in the shuffle model with a distribution-dependent regret of $O\left(\left(\sum_{a\in [k]:\Delta_a>0}\frac{\log T}{\Delta_a}\right)+\frac{k\sqrt{\log\frac{1}{\delta}}\log T}{\eps}\right)$, and a distribution-independent regret of $O\left(\sqrt{kT\log T}+\frac{k\sqrt{\log\frac{1}{\delta}}\log T}{\eps}\right)$, where $T$ is the number of rounds, $\Delta_a$ is the suboptimality gap of the arm $a$, and $k$ is the total number of arms. Our upper bound almost matches the regret of the best known algorithms for the centralized model, and significantly outperforms the best known algorithm in the local model.
\end{abstract}

\section{Introduction}\label{sec:intro}
The \textit{multi-armed bandit (MAB)} problem is a classical sequential decision-making problem in which an agent tries to maximize a cumulative stochastic reward~\cite{thompson1933likelihood, robbins1952some} under uncertainty. This problem, which is applicable to various areas such as recommender systems, online advertising and clinical trials, embodies the well known exploration-exploitation trade-off between learning the environment and acting optimally based on our current knowledge about the environment. 

More formally, in the MAB problem at each time $t=1,\ldots,T$ an agent chooses an arm $i$ from the set $[k]=\{1,\ldots,k\}$ of $k$ arms, and obtains an iid reward $r^t$ drawn from the unknown distribution $R_i$ over $\{0,1\}$ with expectation $\mu_i=\expect{}{R_i}$. Let $a^*=\argmax_{a}\mu_{a}$ be an arm with the largest expected reward, and denote this reward by $\mu^*=\mu_{a^*}$. Let the (suboptimality) gap of an arm $a$ to be the gap between its expected reward and that of $a^*$, i.e., $\Delta_a = \mu^* - \mu_a$. The agent's goal is to maximize the total expected reward, or rather to minimize the expected regret $R(T)=T\cdot \mu^*-\expect{}{\sum_{i=1}^{T}r^t}$ defined to be the expected gap between the algorithm and the optimal algorithm that knows the distributions $R_i$.

In this work we address the privacy in such a setting.
As a motivating example, consider an advertisement system in which 
the server presents to each user an
advertisement $a\in [k]$.
The user then decides whether to click on the 
advertisement or not.
This click decision depends on different private characteristic of the user. The user then reports to the server whether it
clicked on the advertisement (in which case its reward is $r=1$) or not ($r=0$).
 From this example, it is clear that $r$ is private information of the user, and using traditional algorithms for the MAB problem incautiously might leak user-private data.

In order to mathematically alleviate privacy concerns, Dwork et al.~\cite{dwork2006calibrating} defined the notion of \textit{differential privacy (DP)}, which requires that the output of the computation has a limited dependency on any single user's data. Formally, a mechanism $(\eps,\delta)$-DP if for any pair of neighboring inputs (differing by a single user's data), the probability that the mechanism outputs a value in any set $B$ is not different by more than a multiplicative factor of $e^\eps$ and an additive factor of $\delta$. \textit{Differential privacy} has been extensively studied under many different sub-models of privacy. On one end of the spectrum lies the \textit{centralized model} of differential privacy, where the users trust the server with their data, and the liability to protect user privacy lies on the server, who must make sure that any data published externally (e.g., aggregated statistics) respects the privacy constraints. On the other end of the spectrum lies the (strictly stronger) \textit{local model} of differentialy privacy (LDP), where the user privatizes its own data prior to sending it to the server.\footnote{Formally, in the non-interactive setting, a mechanism is $(\eps,\delta)$-LDP if for any two user inputs, the probability that the privatizer sends the server a value in any set $B$ is not different by more than a multiplicative factor of $e^\eps$ and an additive factor of $\delta$.}

Differentially private versions of the MAB problem have been considered in various previous works, where the private information are users' rewards (two neighboring inputs differ by the reward value of a single user), and the algorithm's output is the subsequent arm(s) it selects.
Table~\ref{tab:results} summarizes the best known distribution-dependent regret bounds for the various privacy models, together with our new result in the \textit{shuffle model} which we soon define formally.

\begin{table}
    \caption{Best-known MAB regret upper and lower bounds for various DP models.}
    \label{tab:results}
    \centering
    \begin{tabular}{ll}
    \toprule
    \textbf{Privacy model} & \textbf{Best-known regret upper and lower bounds\tablefootnote{The corresponding distribution-independent regret bounds usually simply replace the $\sum_{a\in [k]:\Delta_a>0}\frac{\log T}{\Delta_a}$ term with $\sqrt{kT\log T}$.}}\\ 
    \midrule
    \textbf{Centralized $(\eps,0)$-DP} & $\Theta\left(\left(\sum_{a\in [k]:\Delta_a>0} \frac{\log T}{\Delta_a}\right)+\frac{k\log T}{\eps}\right)$~\cite{sajed2019optimal,shariff2018differentially}\\ 
    \textbf{Centralized $(\eps,\delta)$-DP} & $O\left(\left(\sum_{a\in [k]:\Delta_a>0} \frac{\log T}{\Delta_a}\right)+\frac{k}{\eps}\right)$~\cite{tossou2016algorithms}\\ 
    \textbf{Local $(\eps,0)$-DP} & $\Theta\left(\frac{1}{\eps^2}\sum_{a\in [k]:\Delta_a>0}\frac{\log T}{\Delta_a}\right)$~\cite{ren2020multi}\tablefootnote{This lower bound can be extended from $(\eps,0)$-LDP to $(\eps,\delta)$-LDP using arguments from Bun et al.~\cite{bun2019heavy} and Cheu et al.~\cite{cheu2019distributed} since we focus on single-round (non-interactive) mechanisms in which the user can only send information to another party once. 
    % Assuming otherwise, since any centralized-DP can be implemented in the shuffle-DP model via cryptographic multiparty computations (MPCs), then they are essentially equivalent. However, MPC-based techniques still have huge costs in terms of computation and communication, and the ongoing necessary communication renders this solution impractical in the real world.\uri{I think we should remove from this footnote everything after ``assuming otherwise...''. It does not feel related here.}
    }\\ 
    % \textbf{Local $(\eps,\delta)$-DP} & $\tilde{\Omega}\left(\sum_{a\in [k]:\Delta_a>0}\frac{\log T}{\eps^2\Delta_a}\right)$ LB \\ 
    \textbf{Shuffle $(\eps,\delta)$-DP (ours)} & $O\left(\left(\sum_{a\in [k]:\Delta_a>0}\frac{\log T}{\Delta_a}\right)+\frac{k\sqrt{\log\frac{1}{\delta}}\log T}{\eps}\right)$\\
    \bottomrule
    \end{tabular}

\end{table}
% Considering the \textit{central Model} of differential privacy, for $(\eps,0)$-DP, Sajed and Sheffet~\cite{sajed2019optimal} used Successive Elimination to achieve a distribution-dependent regret of $O\left(\left(\sum_{a\in [k]:\Delta_a>0} \frac{\log T}{\Delta_a}\right)+\frac{k\log T}{\eps}\right)$ (matching a previous lower bound from~\cite{shariff2018differentially}), and a distribution-independent regret of $O\left(\sqrt{Tk\log T} + \frac{k\log T}{\eps}\right)$. For the more general $(\eps,\delta)$-DP, Tossou and Dimitrakakis~\cite{tossou2016algorithms} shave the $\log T$ factor, and achieve a distribution-dependent regret of $O\left(\left(\sum_{a\in [k]:\Delta_a>0} \frac{\log T}{\Delta_a}\right)+\frac{k}{\eps}\right)$ which has negligible additional regret due to privacy.
% Considering the \textit{local model} of differential privacy, for $(\eps,0)$-LDP, Ren et al.~\cite{ren2020multi} achieve a distribution-dependent regret of $O\left(\sum_{a\in [k]:\Delta_a>0}\left(\frac{\log T}{\eps^2\Delta_a}+\Delta_a \right)\right)$ which matches their lower bound of $\Omega\left(\sum_{a\in [k]:\Delta_a>0}\frac{\log T}{\eps^2\Delta_a}\right)$, and achieve a distribution-independent regret of $O(\sqrt{kT\log T}/\eps)$. 
With real-world algorithms gradually moving away from the \textit{centralized model} of privacy, the immediate question is \enquote{can we design a private algorithm for MAB with privacy guarantees which are similar to local DP, but with similar regret to centralized DP (without a multiplicative $1/\eps^2$ factor)?}.

To address the inevitable gap between the local and centralized models, which is in fact common in the literature of differential privacy, the alternative \textit{shuffle model}~\cite{bittau2017prochlo,cheu2019distributed,erlingsson2019amplification} explores the space in between the local and centralized models by introducing a trusted shuffler that receives user messages and permutes them (i.e., disassociates a message from its sender) before they are delivered to the server. 
For privacy analysis, we assume that the shuffle is  perfectly secure, i.e., its output contains no information about which user generated each of the messages. This is traditionally achieved by the shuffler stripping implicit metadata from the messages (e.g., timestamps, routing information), and frequently forwarding this data to remove time and order information.
The shuffle model ensures that sufficiently many reports are collected in each round so that any one report can hide in a shuffled batch. In order to apply the shuffle model to the MAB problem in the context of advertisements, we divide the algorithm into batches, where before each batch we decide on the fly its size $m$, and then present the $m$ next users the same advertisement $a$, and finally apply a private shuffle model mechanism to their rewards to communicate reward aggregate information to the server.\footnote{We remark that a given user does not know in advance the size of the batch in which it participates, since this size depends on the algorithm's run.}

A constantly growing body of work presents new and improved mechanisms in the shuffle model for basic statistical tasks~\cite{ghazi2019scalable, erlingsson2019amplification, erlingsson2020encode, balle2019privacy}, such as \textit{private binary summation}, in which the server must privately approximate the sum of a collection of values $x_1,..., x_m \in \{0, 1\}$ held by the $m$ users in the shuffle. For \textit{private binary summation}, the optimal achievable errors in the central, local and shuffle model are $\tilde{\Theta}(1/\eps)$~\cite{dwork2006calibrating}, $\tilde{\Theta}(\sqrt{m}/\eps)$~\cite{beimel2008distributed,chan2012optimal} and $\tilde{\Theta}(1/\eps)$~\cite{cheu2019distributed} respectively, where the $\tilde{\Theta}(\cdot)$ hides poly-logarithmic terms. Similar errors hold for the \textit{private summation} problem which approximates the sum of real values in $[0,1]$.

\subsection{Our contributions}
To the best of our knowledge, our work is the first to consider the MAB problem under the shuffle model of differential privacy. In order to support the online nature of the MAB problem, we consider a variant of the shuffle model. As opposed to the classical shuffle model, in which the shuffle size is unbounded and the mechanism runs only once, we continuously run shuffle mechanisms for many disjoint batches of users who can only afford a single round of communication.\footnote{Beimel et al.~\cite{beimel2020round} showed that every centralized-DP mechanism can be emulated in the shuffle model in two (communication-intensive) rounds of communication, however these results are not applicable to our setting 
since we assume that each online user participates only in one shuffle  and
then disappears.} 
%\uri{I removed the following sentence, because it's hard to understand it at this point: ``and we intentionally set the exact size of the batch before it comes.''}

We consider the paradigm where the server controls the different users who are cooperative and communicate only with the server. We give a rigorous definition of \textit{shuffle differential privacy (SDP)} for the multi-armed bandit problem (assuming binary rewards), and give and prove the first two such algorithms. Our algorithms \textit{Shuffle Differentially Private Arm Elimination (\SDPAE)} and \textit{Variable Batch Shuffle Differentially Private Arm Elimination (\VBSDPAE)} are both based on the well-known \textit{arm elimination (\AE)} algorithm, using consecutive batches of users and together with an SDP private binary summation mechanism.

We show that the simpler but weaker \SDPAE achieves a distribution-dependent regret of 
$O\left(\left(\sum_{a\in [k]:\Delta_a>0}\frac{\log T}{\Delta_a}\right)+\frac{k\log\frac{1}{\delta}}{\eps^2}\right)$, and a distribution-independent regret of $O\left(\sqrt{kT\log T}+\frac{k\log\frac{1}{\delta}}{\eps^2}\right)$.\footnote{In this paper, we assume that $T$ is known apriori. Otherwise, we can apply standard doubling arguments to get roughly the same results.}

We then describe \VBSDPAE, a generalization of \SDPAE to exponentially growing batch sizes, and prove it has a distribution-dependent regret of $O\left(\left(\sum_{a\in [k]:\Delta_a>0}\frac{\log T}{\Delta_a}\right)+\frac{k\sqrt{\log\frac{1}{\delta}}\log T}{\eps}\right)$, and a distribution-independent regret of $O\left(\sqrt{kT\log T}+\frac{k\sqrt{\log\frac{1}{\delta}}\log T}{\eps}\right)$. 

Note that, compared to the local model (Ren et al~\cite{ren2020multi}), the regret of
both \SDPAE and \VBSDPAE is improved, by having the dependency on $1/\eps$ be additive rather than multiplicative.
%
%shave off the inevitable $\frac{1}{\eps^2}$ regret multiplicative factor from the local model (Ren et al~\cite{ren2020multi}), and 
In addition, \VBSDPAE almost matches the regret of the best known algorithms for the centralized model, that is the distribution-dependent regret of $O\left(\left(\sum_{a\in [k]:\Delta_a>0} \frac{\log T}{\Delta_a}\right)+\frac{k}{\eps}\right)$ of Tossou and Dimitrakakis~\cite{tossou2016algorithms}.

% All the missing proofs appear in the appendix.

\subsection{Related work}
% \paragraph{Multi-armed bandit}
% In this problem, we have $k$ arms, and each one has an unknown latent distribution, and a pull returns a random reward from it. Denote $R_a$ the distribution from each arm and $\mu_a=E[R_a]$ the mean reward. We assume that rewards are independent across arms and time. Define the gap of an arm a as $\Delta_a = \max_{a'}\mu_{a'} - \mu_a$. Denote $N_a^t$ the number of pulls of $a$ until time $t$, and $A^t$ the arm at time $t$.
The differentially private MAB problem has been considered in many previous works~\cite{jain2012differentially,guha2013nearly,mishra2015nearly}. Shi and Shen~\cite{shi2021federated} and Dubey and Pentland~\cite{dubey2020differentially} studied MAB and linear bandits respectively in the federated setting. Zheng et al.~\cite{zheng2020locally} studied contextual bandits with LDP.
Batched MAB with a predetermined number of batches was studied in~\cite{esfandiari2019regret,gao2019batched}.

For the private summation problem, Cheu et al.~\cite{cheu2019distributed} gave unbiased $(\eps,\delta)$-SDP mechanisms over binary inputs and real inputs in $[0,1]$, with error roughly $\sqrt{\log ({1}/{\delta}})/\eps$ and $\log ({1}/{\delta})/\eps$, respectively. In several works of Balle et al.~\cite{balle2019privacy,balle2020private,balle2019differentially}, they gave a biased $(\eps,\delta)$-SDP mechanism for real inputs in $[0,1]$ with similar error and a constant number of messages, and a single-message $(\eps,\delta)$-SDP mechanism with optimal error.

\section{Background and preliminaries}\label{sec:prelim}

\subsection{Shuffle-model privacy}
In the well-studied setting of shuffle model privacy, there are $m$ users, each with data $x_i \in X$. Each user applies some encoder $E : X \to Y^*$ to their data and sends the messages $(y_{i,1},...,y_{i,p}) = E(x_i)$ to a shuffler $S : Y^* \to Y^*$. The shuffler then shuffles all the messages $y_{i,j}$ from all the users, and outputs them in a uniformly random order to an analyzer $A : Y^* \to Z$ to estimate some function $f(x_1,..., x_m)$. Thus, the mechanism $M$ consists of the tuple $(E, S,A)$. We say that such a mechanism $M$ is \textit{$(\eps, \delta)$-shuffle differentially private} (or \textit{$(\eps, \delta)$-SDP} for short) if the output of $S$ is $(\eps, \delta)$-differentially private, or more formally:
A mechanism $M=(E, S,A)$ is $(\eps,\delta)$-SDP if for any pair of inputs $\{x_i\}_{i=1}^{m}$ and $\{x_i'\}_{i=1}^{m}$ which differ in at most one value, we have for all $B\subseteq Y^*$:
$$P(S(\cup_{i=1}^{m}E(x_i))\in B)\leq e^\eps\cdot P(S(\cup_{i=1}^{m}E(x_i'))\in B)+\delta.$$

In the mechanism used in this paper, $E$ outputs the user's reward bit together with a set of random bits, and $A$ sums all these bits and debiases the result to get an unbiased estimate of the sum of the users' rewards.

% \paragraph{Batched policies.}
% The \textit{batched multi-armed bandit} problem is an extension of the multi-armed bandit involving batches of $m\in \naturals$ agents. In every round $t\in [T/m]$ (we assume $m$ divides $T$), we sample a batch of $m$ new agents, and the server selects a single arm $a^t\in [k]$ which all agents $i\in [m]$ take, and each agent $i$ obtains an independent reward $r_{a,i}^t$ by sampling the arm $a^t$.

% The objective is to minimize the cumulative (pseudo) regret, which is the difference between the sum of the  rewards accumulated by the algorithm and the sum of the rewards of the optimal algorithm that apriori knows the best arm $a^*$ which has a mean reward of $\mu_{a^*}$.
% For any arm $a$ and round $t$, let $N_a^t$ be the number of rounds we selected the arm $a$ until round $t$, $N_a^t=\sum_{s=1}^{t}1_{a^s=a}$. 
% Let $a^t$ be the arm we took in the $t$'th round. After $t$ rounds, the (expected) reward is $\sum_{s=1}^t\sum_{i=1}^{m} \expect{}{\mu_{a^t}}=\expect{}{\sum_{a\in \{1,\ldots,k\}} N_{a}^{t} m \Delta_a}.$
% The (pseudo) regret of the algorithm for $T$ users (i.e., for $T/m$ rounds) is formally defined to be
% $$R(T) = T\cdot \mu^* - \sum_{s=1}^{T/m}\sum_{i=1}^{m} \expect{}{\mu_{a^t}} = \expect{}{\sum_{a\in \{1,\ldots,k\}} N_{a}^{T/m} m \Delta_a}.$$
%  Our aim is to devise an algorithm with a small regret.

\subsection{Shuffle-model MAB}
Algorithms which are private in the shuffle model 
typically apply the mechanism $M$ once over a set of $m$ users. Here we study the MAB problem which is an online problem, often deployed in real-world applications and with users which are end-devices such as cellphones with a possibly limited or unreliable internet connection. Hence, to adapt the MAB problem to the shuffle model, we batch sequences of consecutive users, and assume that each user can afford a single round of communication, and is never selected more than once.

\paragraph{Model and objective}
%\haim{I rewrote below a bit to address Uri's point}
The \textit{shuffle-model MAB} setting involves repeating the following process until the $T$'th player pulls its arm:
\begin{enumerate}
    \item The server selects a batch size $m$,
    a \textit{batch} of $m$ random fresh new users, and an $m$-user single-round SDP mechanism $M$.\footnote{Note that always selecting $m=1$ reduces this setting to the Local model MAB.}
    It then picks  an arm
 $a\in [k]$ that all $m$ users of the batch  pull. (For concreteness think that the server picks an ad $a\in [k]$ and sends it to the a random batch of $m$ users.)
    \item 
    Each user $i$ determines binary reward from pulling the arm $a$.\footnote{For simplicity, we assume that the rewards are binary. However, our algorithms and proofs naturally extend to the real $[0,1]$ reward setting, by replacing our private binary summation mechanism (defined later, see Appendix~\ref{sec:privateSummation}) with a private summation mechanism for real numbers in $[0,1]$ with similar guarantees.}
(For concreteness think 
that each user decides whether to click on the ad or not. This defines the reward related to user $i$, which is $r_i=1$ if it clicks the ad and 
    $r_i=0$ otherwise.)
  
    Since our $m$ users are random, these rewards are a sample of $m$ independent rewards from 
    the distribution $R_a$ associated with arm $a$.
    \item The server computes $M(\{r_i\}_{i\in batch})$ using the rewards $r_i$.    
\end{enumerate}

The objective is to minimize the (pseudo) regret, which is the expected difference between the sum of the rewards accumulated (over all the users) by the algorithm and the sum of the rewards of the optimal algorithm that apriori knows an arm with the largest expected reward $a^*$. Let $\mu_a=\expect{}{R_a}$ be the \textit{expected reward} (or simply \textit{mean}) of the arm $a$, let $\mu^*$ denote the expected reward of $a^*$, let $\Delta_a = \mu^* - \mu_a$ be the (suboptimality) gap of the arm $a$ which quantifies the gap between its expected reward and that of $a^*$, and let $N_a$ be the random variable which counts the total number of times the arm $a$ was pulled during the run of the algorithm. 

Formally, the (pseudo) regret of the algorithm for $T\in \naturals$ users is defined to be 
$$R(T) = T\cdot \mu^* - \expect{}{\sum_{i=1}^{T} r_{i}} =\expect{}{\sum_{a\in [k]}N_a \Delta_{a}}.$$
% where $B$ denotes the random set of batches of users used by the algorithm during its run.

\paragraph{Privacy}
Since the private data of each user is its reward, the appropriate adaptation of shuffle model privacy for the multi-armed bandit problem is as follows.
An algorithm for the multi-armed bandits problem is \textit{$(\eps,\delta)$-shuffle differentially private (SDP)} if for any batch of users, the shuffle mechanism that we apply over them in step 3 is $(\eps, \delta)$-SDP with respect to the rewards of the users (as we recall, the reward of a user -- whether it clicked on an ad or not -- depends on its private features). Formally, for every batch we run a shuffle mechanism where the $m$ users are the users of the current batch, and the data $x_i$ of each user is its reward $r_i$, and we require that each such mechanism is $(\eps, \delta)$-SDP.

% \red{TODO: add the formalities from SDP here.}
% Formally, for a batch of $m$ users, we say that a pair of reward sequences $(r_i)_{i\in [m]}$ and $({\tilde{r}}_i)_{i\in [m]}$ are neighbors if $\forall i'\ne i,~r_i = {\tilde{r}}_i$. We say that an algorithm for the multi-armed bandits problem is \textit{$(\eps,\delta)$-shuffle differentially private (SDP)} if for any batch of users, any pair of reward sequences $(r_i)_{i\in [m]}$ and $({\tilde{r}}_i)_{i\in [m]}$, and any $B\subseteq Y^*$, the mechanism $M=(E,S,A)$ used in step 3.\ satisfies $$P(S(\cup_{i=1}^{m}E(r_i))\in B)\leq e^\eps\cdot P(S(\cup_{i=1}^{m}E({\tilde{r}}_i))\in B)+\delta.$$

\subsection{Concentration bounds} 
We use the following standard definitions of Sub-Gaussian random variables and Hoeffding's inequality.
\begin{definition}[Sub-Gaussian random variable]\label{def:subgaussian}
A random variable $X$ with mean $\mu$ is called \textit{sub-Gaussian with variance $\sigma^2$}, i.e., $X\sim SG(\sigma^2)$ if:
% $$\forall t>0,\max\left({P(X-\mu\geq t),P(X-\mu\leq -t)}\right)\leq \exp\left(\frac{-t^2}{2\sigma^2}\right).$$
$$\forall \lambda\in \reals,~\expect{}{\exp(\lambda(X-\mu))}\leq e^{\lambda^2\sigma^2/2}.$$
\end{definition}
An equivalent definition shows that if $\forall t>0,\max\left({P(X-\mu\geq t),P(X-\mu\leq -t)}\right)\leq \exp\left(\frac{-t^2}{2\sigma^2}\right)$, then $X$ is sub-Gaussian with variance $\sigma^2$ (up to constant factor).
It is well known that if $X_i\sim SG(\sigma_i^2)$ are independent random variables for $i=1,\ldots,n$, then $\sum_{i=1}^{n}X_n \sim SG(\sum_{i=1}^{n}\sigma_i^2)$ and for any $a,b>0$, $a\cdot X_1+b\sim SG(a^2\cdot \sigma_1^2)$. A bounded random variable $X\in[a,b]$ is $SG((b-a)^2/4)$.

Sub-Gaussian random variables satisfy the following concentration bound,
\begin{lemma}[Hoeffding's inequality~\cite{hoeffding1994probability}]\label{lma:hoeffding}
    Let $\{X_i\}_{i=1}^n$ be independent $SG(\sigma^2)$ random variables,
    %let $a,b,\sigma\in \reals$, and let $S = \sum_i x_i$. Then, if $\forall i,P\left(x_i\in [a,b]\right)=1$, 
    then
    %$P\left(\abs{S-\expect{}{S}}\geq t\right)\leq 2\exp\left(\frac{-2t^2}{n(b-a)^2}\right),$
    %and if $\forall i,x_i\sim SG(\sigma^2)$, then
    %$P\left(\abs{S-\expect{}{S}}\geq t\right)\leq 2\exp\left(\frac{-t^2}{2n\sigma^2}\right).$
    $P\left(\abs{\sum_{i=1}^n X_i-\expect{}{\sum_{i=1}^n X_i}}\geq t\right)\leq 2\exp\left(\frac{-t^2}{2n\sigma^2}\right).$
\end{lemma}

\section{Differentially private MAB in the shuffle model}
In this paper, we use the shuffle model to give a private solution to the multi-armed bandit problem, attaining similar privacy guarantees to that of the local privacy model (LDP), without sacrificing the regret. That is, our algorithm almost matches the best known regret in the centralized model of differential privacy. 
% While it is possible to do so through general black-box reductions from mechanisms which respect centralized differential privacy to mechanisms which respect shuffle model privacy, they are usually multi-round computations and are extremely resource-intensive. 
% Hence, 
Our algorithms rely on the fact that  algorithms for the multi-armed bandit
problem take decisions based on 
sums of rewards received from the users. Hence, we rely on a particularly efficient and accurate mechanism for private binary summation in the shuffle model as a building block in our algorithms. Specifically, for any $\eps,\delta\in (0,1)$, let $M_{sum}$ be a private binary summation mechanism, which for any number of users (batch size) $m\in \naturals$, is $(\eps,\delta)$-SDP, unbiased, and has an error distribution which is independent of the input, and is sub-Gaussian with variance $\sigma_{\eps,\delta}^2=O\left(\frac{\log \frac{1}{\delta}}{\eps^2}\right)$.\footnote{Note that our methods and algorithms should work similarly if they were built over other common algorithms for the MAB problem in which batching makes sense, such as the \UCB algorithm.} Note that our notation of $M_{sum}$ does not include $\eps$ and $\delta$ which will always be clear from context or inherited from the algorithm which runs $M_{sum}$.
The challenge is to combine $M_{sum}$ with the well studied \textit{arm elimination (\AE)} MAB algorithm to get an SDP algorithm for the MAB problem with almost optimal regret.\footnote{For completeness, in Appendix~\ref{sec:privateSummation} we give a complete description and a proof of such a mechanism $M_{sum}$.}

\subsection{\SDPAE: Shuffle Differentially Private Arm Elimination}
\label{sec:sdpMAB}

% We describe the method which we use to obtain differential privacy in the shuffle model. 
We base our algorithm on the (non-private) \textit{arm elimination (\AE)} algorithm for the MAB problem, which informally maintains a set of viable arms (initially set to be $[k]$), and each phase pulls the set of viable arms sequentially. Once a phase ends, we search for arms which are noticeably suboptimal in comparison to some other arm, and we eliminate them from the set of viable arms.

To adapt \AE to the shuffle model, in each phase $t$ each arm is pulled not once, but rather by a whole batch of users. Once all the users in the batch pull the arm and receive their reward, we apply the private binary summation mechanism $M_{sum}$ to the batch's rewards, which gives the server an unbiased but noisy estimate of the sum of rewards in the batch. This estimate has two sources of error, which we account for when we compute the upper confidence bound -- the empirical error due to sampling the reward function of the arm, and the error due to the private binary summation mechanism $M_{sum}$.

\subsubsection{Algorithm outline}
In the algorithm below which we call \textit{Shuffle Differentially Private Arm Elimination (\SDPAE)}, we update the estimate of the mean reward of each arm $a$ after every batch of users who sample $a$.
The algorithm works in phases, and maintains a set of viable arms initially set to be $[k]$. In each phase, for every viable arm we have a single batch of users sampling it.
At the end of phase $t$, we denote by $\hat{S}_a^t$  the noisy estimate of the cumulative sum of the rewards from all previous  samples of $a$ in the phases $1,\ldots,t$, and denote by $N_a^t$ the total number of previous samples of arm $a$ in the phases $1,\ldots,t$. The natural estimate for the mean reward of the arm $a$ (denoted by $\hat{\mu}_a^t$) is therefore $\hat{\mu}_a^t=\hat{S}_a^t/N_a^t$. We then calculate the upper and lower confidence bounds $UCB_a^t$ and $LCB_a^t$ respectively of each viable arm $a$ after each phase $t$ using a specific bound which takes into account both sources of error. We finally eliminate any remaining arm with an upper confidence bound which is strictly smaller than the lower confidence bound of some other arm. Algorithm~\ref{alg:batchae} consolidates the algorithm presentation above.

\medskip

\begin{algorithm}[H]\label{alg:batchae}
\SetAlgoLined
\textbf{Input:} privacy parameters $\eps$ and $\delta$, batch size $m$ and horizon $T$.\\
%\jay{I removed $M_{sum}$ from here since it's defined previously.}\\
\textbf{Initialize:} $\hat{S}_a^0=0$, $N_a^0=0$ and $\hat{\mu}_{a}^0=0$ for every $a\in [k]$\;
Let $\sigma_{\eps,\delta}^2=O\left(\frac{\log \frac{1}{\delta}}{\eps^2}\right)$ be the sub-Gaussian variance of the error distribution of $M_{sum}$\;
Let $V=[k]$ denote the set of viable arms\;
 \For{phase $t\gets1,2,\ldots$}{
    \For{arm $a\in V$}{
        \For{each new user $i\gets1$ \KwTo $m$}{
            User $i$ pulls the arm $a$ and observes reward $r_{a,i}^t$\;
            If total arm samples in current algorithm run is $T$, exit\;
        }
        \textbf{Communication:} Perform private binary summation 
        $Z_a^t \gets M_{sum}\left(\{r_{a,i}^t\}_{i=1}^{m}\right)$\;
        \textbf{Server update:} Update $\hat{S}_{a}^{t} \gets \hat{S}_{a}^{t-1} + Z_a^t $,
        $N_{a}^{t} \gets N_{a}^{t-1} + m$, and finally $\hat{\mu}_{a}^{t} \gets \hat{S}_{a}^{t}/N_{a}^{t}$\;
    }
   \textbf{Confidence bounds:} For each arm $a$, calculate $I_a^t\gets\left(\frac{2\sqrt{t}\sigma_{\eps,\delta}}{N_a^t}+\frac{1}{\sqrt{N_a^t}}\right)\cdot\sqrt{2\log T}$, and the upper and lower confidence bounds $UCB_a^t \gets \hat{\mu}_{a}^{t}+I_a^t$ and $LCB_a^t \gets \hat{\mu}_{a}^{t}-I_a^t$\;
   \textbf{Elimination:} remove all arms $a$ from $V$ such that $UCB_a^t < \max_{a'\in S}LCB_{a'}^t$\;
 }
 \caption{\SDPAE (Shuffle Differentially Private Arm Elimination)}
\end{algorithm}
\subsubsection{Analysis}
The privacy is trivial, since each batch we use the $(\eps,\delta)$-SDP mechanism $M_{sum}$. We now focus on regret.

Theorem~\ref{thm:batchAERegret} gives a bound on the regret of \SDPAE as a function of the batch size $m$. We follow a somewhat standard regret bound analysis for arm elimination, comprising two parts. The first part uses Hoeffding's inequality to derive a high probability bound on $\abs{\hat{\mu}_a^t-\mu_a}$, the error between the empirical average reward and the true mean reward of the arm $a$ at a given phase $t$. The second part uses this bound to bound the expected number of times we sample each suboptimal arm $a$, and summing over all suboptimal arms we get a bound on the regret.

\begin{theorem}\label{thm:batchAERegret}
    %Let $M_{sum}$ be private binary summation mechanism, which is unbiased, has an error distribution which is independent of the input, and is sub-Gaussian with variance $\sigma_{\eps,\delta}^2$. Then the algorithm we get from \SDPAE instantiated with $M_{sum}$
    The algorithm \SDPAE is $(\eps,\delta)$-SDP, and has a distribution-dependent regret of $O\left(\sum_{a\in [k]:\Delta_a>0}\left(\frac{\log T}{\Delta_a}+\frac{\sigma_{\eps,\delta}^2\log T}{m\Delta_a}+m\Delta_a\right)\right)$, and a distribution-independent regret of $O\left(\sqrt{\left(1+\frac{\sigma_{\eps,\delta}^2}{m}\right)kT\log T}+mk\right)$.
\end{theorem}
\begin{proof}
First, note that for any arm $a$, and phase $t$, we have
%First, note that 
$\hat{\mu}_a^t=\frac{\sum_{s=1}^{t}M_{sum}\left(\{r_{a,i}^s\}_{i=1}^{m}\right)}{N_a^t}$, where 
$N_a^t=m\cdot t$.

We define the \textit{clean event} $C:=\left\{\forall a\in[k],\forall t\in[T]~~\abs{\hat{\mu}^t_a-\mu_a}\leq I_a^t\right\},$ where $I_a^t:=\left(\frac{2\sqrt{t}\sigma_{\eps,\delta}}{N_a^t}+\frac{1}{\sqrt{N_a^t}}\right)\cdot \sqrt{2\log T}$ is a confidence bound interval. We now show that the clean event $C$ occurs with high probability, that is $P\left(C\right)\geq1-4T^{-2}$, and after that assume the event $C$ to simplify our analysis.

Indeed, for each arm $a$, we imagine both a reward tape of length $1\times T$, with each cell independently sampled from the distribution $D_a$ of arm $a$, and a private-binary-summation-error tape of length $1\times T$, with each cell independently sampled from the distribution of (the additive) error of the private binary summation mechanism $M_{sum}$ for $m$ users.

We assume that in the $j$'th time a given arm $a$ is pulled by the algorithm, its reward is taken from the $j$'th cell in this arm's reward tape, and similarly the $j$'th time we compute a private binary sum over rewards of a batch of users who pulled $a$, the (additive) error is taken from the $j$'th cell in the arm's private-binary-summation-error tape.\footnote{Here we rely on the fact that the distribution of the error of $M_{sum}$ is independent of the input.}$^,$\footnote{Note that sizes of both tapes have been chosen conservatively to be of size $T$. We never pass the end of any these tapes, since there are at most $T$ users in total, and at most $T$ batches (actually roughly $T/m$ in this case), and we may not use them all.}

Let $t\in [T]$ and $a\in [k]$, and let $\hat{v}_a^t$ be the approximated reward of arm $a$ that the algorithm would have held at the end of phase $t$ using the concrete values in the tapes defined above, 
that is $ \hat{v}^t_a =
\frac{\sum_{s=1}^{t}d_s+\sum_{i=1}^{N_a^t}e_{i}}{N_a^t}$, where $d_{s}$ is the $s$'th cell of the private-binary-summation-error tape of $a$, and the $\{e_{i}\}_{i=1}^{N_a^t}$ are the total $N_a^t=m\cdot t$ cells of the reward tape of $a$ that we have used until the end of the $t$'th phase.

Our aim is to bound the term $\abs{\hat{v}^t_a-\mu_a}=\frac{\sum_{s=1}^{t}d_s+\sum_{i=1}^{N_a^t}(e_{i}-\mu_a)}{N_a^t}$. We first bound the first sum in the nominator, then bound the second sum in the nominator, and finally combine the bounds to get a bound for $\abs{\hat{v}^t_a-\mu_a}$.

To bound the $d_{s}$'s sum, we apply Hoeffding's inequality (Lemma~\ref{lma:hoeffding}) for a sum of $n\gets t$ random variables which are sub-Gaussian with variance $\sigma_{\eps,\delta}^2$ and have zero mean (since $M_{sum}$ is unbiased) to get,
\begin{align}
   P\left(\abs{\sum_{s=1}^{t} d_{s}}\leq 2\sigma_{\eps,\delta}\cdot \sqrt{2t\log T}\right)
    &\geq 1-2\exp\left(-\left(4\sigma_{\eps,\delta}^2\cdot 2t\log T\right)/\left(2t\sigma_{\eps,\delta}^2\right) \right)\nonumber\\
    &= 1-2T^{-4}.\label{eq:dsbound}
\end{align}

To bound the $f_i=(e_{i}-\mu_a)$'s sum, we apply Hoeffding's inequality (Lemma~\ref{lma:hoeffding}) for the sum of $n\gets N_a^t$ random variables $f_i$, each sub-Gaussian with variance 1/4 (since it is bounded in the interval $[-\mu_a,1-\mu_a]$ of size 1), and with zero mean (since by its definition $\expect{}{e_i-\mu_a}=\expect{}{e_i}-\mu_a=0$), to get that 
\begin{align}
   P\left(\abs{\sum_{i=1}^{N_a^t}(e_{i}-\mu_a)}\leq \sqrt{2N_a^{t}\log T}\right)
    &\geq1- 2\exp\left(-(2N_a^t\log T)/(2N_a^t/4)\right)=1-2T^{-4}.\label{eq:eibound}
\end{align}

Applying a union bound and the triangle inequality on Equation~\eqref{eq:dsbound} and Equation~\eqref{eq:eibound} gives that
$$P\left(\abs{\sum_{s=1}^{t} d_{s}+\sum_{i=1}^{N_a^t}(e_{i}-\mu_a)}\leq \left(2\sqrt{t}\sigma_{\eps,\delta}+\sqrt{N_a^t}\right)\cdot \sqrt{2\log T}\right)\geq1-4T^{-4},$$
which by the definition of $\hat{v}^t_a$ and $I_a^t$ means that
\begin{align}
P\left(\abs{\hat{v}^t_a-\mu_a}\leq 
% \left(\frac{2\sqrt{t}\sigma_{\eps,\delta}}{N_a^t}+\frac{1}{\sqrt{N_a^t}}\right)\cdot \sqrt{2\log T}
I_a^t
\right)\geq1-4T^{-4}.\label{eq:VaMuaBound}
\end{align}

Since in the analysis above $t$ and $a$ are arbitrary, Equation~\eqref{eq:VaMuaBound} holds for every $t\in [T]$ and $a\in [k]$. Thus, we take a union bound over all arms $a\in [k]$ (assuming $k\leq T$) and all $t\in [T]$, to conclude that 
\begin{align}
P\left(\forall a\in [k], \forall t\in [T]~\abs{\hat{v}^t_a-\mu_a}\leq
% \left(\frac{2\sqrt{t}\sigma_{\eps,\delta}}{N_a^t}+\frac{1}{\sqrt{N_a^t}}\right)\cdot \sqrt{2\log T}
I_a^t
\right)\geq1-4T^{-2}.\label{eq:allBound}
\end{align}
Since the event in the probability above in Equation~\eqref{eq:allBound} is precisely the event that $C$ holds for a run of the algorithm using the randomness in the tapes as defined above (by the definitions of $\hat{v}^t_a$ and $\hat{\mu}^t_a$), we get that $$P\left(C\right)\geq1-4T^{-2}.$$  

For the regret analysis, we assume the clean event $C$. Consider a suboptimal arm $a$ such that $\Delta_a=\mu^*-\mu_a>0$, and consider the last phase $t_0$ following which we did not remove the arm $a$ yet (or the last phase if $a$ remains active to the end). Since we assumed the clean event, an optimal arm $a^*$ cannot be disqualified, and since $a$ is not yet disqualified, the confidence intervals of the arms $a$ and $a^*$ at the end of the $t_0$'s phase must overlap. Therefore, 
\begin{align}
\Delta_a=\mu^*-\mu_a\leq 2(I_a^{t_0}+I_{a^*}^{t_0}) =4I_a^{t_0}=\left(\frac{8\sqrt{t}\sigma_{\eps,\delta}}{N_a^{t_0}}+\frac{4}{\sqrt{N_a^t}}\right)\cdot \sqrt{2\log T},\label{eq:deltaaBound}
\end{align}
where the third step follows since $a$ and $a^*$ were sampled using identical batch sizes throughout the algorithm, so at the end of the ${t_0}$'th phase, $N_a^{t_0}=N_{a^*}^{t_0}$ and therefore $I_a^{t_0}=I_{a^*}^{t_0}$, and the last step follows by the definition of $I_a^{t_0}$.

Observe that if $N_a^{t_0}> \frac{128\log T}{\Delta_a^2}$ then $\frac{4\cdot \sqrt{2\log T}}{\sqrt{N_a^{t_0}}}<\frac{\Delta_a}{2}$, and if $N_a^{t_0}>\frac{(16\sigma_{\eps,\delta}\cdot \sqrt{2\log T})^2}{m\Delta_a^2}$ then $\frac{8\sqrt{{t_0}}\sigma_{\eps,\delta}\cdot \sqrt{2\log T}}{N_a^{t_0}}=\frac{8\sqrt{N_a^{t_0}/m}\sigma_{\eps,\delta}\cdot \sqrt{2\log T}}{N_a^{t_0}}=\frac{8\sigma_{\eps,\delta}\cdot \sqrt{2\log T}}{\sqrt{mN_a^{t_0}}}<8\sigma_{\eps,\delta}\cdot \sqrt{2\log T}\cdot \frac{\Delta_a}{16\sigma_{\eps,\delta}\cdot \sqrt{2\log T}}=\frac{\Delta_a}{2}$, so their sum is $<\Delta_a$ in contradiction to Equation~\eqref{eq:deltaaBound}. Hence, $N_a^{t_0}\leq \max\left( \frac{128\log T}{\Delta_a^2},\frac{512\sigma_{\eps,\delta}^2\cdot \log T}{m\Delta_a^2}\right)$.
Therefore the total regret on arm $a$ is
\begin{align}
R_a&\leq \Delta_a\cdot (N_a^{t_0}+m)\leq \max\left( \frac{128\log T}{\Delta_a},\frac{512\sigma_{\eps,\delta}^2\cdot \log T}{m\Delta_a}\right)+m\Delta_a\nonumber\\
&\leq \frac{128\log T}{\Delta_a}+\frac{512\sigma_{\eps,\delta}^2\cdot \log T}{m\Delta_a}+m\Delta_a,\label{eq:RaBound}
\end{align}
where the first step follows since since the arm $a$ is eliminated following phase ${t_0}+1$ (or if $t_0$ is the last phase, then we finish and don't sample $a$ after it) of batch size $m$ and is subsequently never pulled, and the second step follows by previous bound on $N_a^{t_0}$ and since $N_a^{{t_0}+1}=N_a^{t_0}+m$.
We sum up the regret over all arms, to obtain a bound for the total regret denoted by $R$:

$$R=\sum_{a\in [k]:\Delta_a>0}R_a\leq \sum_{a\in [k]:\Delta_a>0}\left(\frac{128\log T}{\Delta_a}+\frac{512\sigma_{\eps,\delta}^2\cdot \log T}{m\Delta_a}+m\Delta_a\right).$$

To complete the analysis, we argue that the bad event in which $C$ does not hold contributes a negligible amount to the expected regret $R(T)$. Indeed,

\begin{align}
    R(T)&=\expect{}{R\mid C}\cdot P(C)+\expect{}{R \mid \bar{C}}\cdot P(\bar{C})\nonumber\\
    &\leq \sum_{a\in [k]:\Delta_a>0}\left(\frac{128\log T}{\Delta_a}+\frac{512\sigma_{\eps,\delta}^2\cdot \log T}{m\Delta_a}+m\Delta_a\right)+T\cdot4T^{-2}\nonumber\\
    &= O\left(\sum_{a\in [k]:\Delta_a>0}\left(\frac{\log T}{\Delta_a}+\frac{\sigma_{\eps,\delta}^2\log T}{m\Delta_a}+m\Delta_a\right)\right),\label{eq:expectedRegret}
\end{align}
where the first step follows by the law of total expectation, and the second step follows since the regret is at most $T$, and by the previous bound on $P(C)$.

Now for the distribution-independent bound, assume the clean even $C$, and let $\gamma>0$ be a threshold whose exact value we will set later. We group the arms $a$ based on if $\Delta_a<\gamma$ or not, to get
\begin{align*}
    R&=\sum_{a\in [k]\mid \Delta_a<\gamma} R_a
+\sum_{a\in [k]\mid \Delta_a\geq\gamma}R_a\\
    &\leq T\cdot \gamma
+\sum_{a\in [k] \mid \Delta_a\geq\gamma}\left(\frac{128\log T}{\Delta_a}+\frac{512\sigma_{\eps,\delta}^2\cdot \log T}{m\Delta_a}+m\Delta_a\right)\\
    &\leq T\cdot \gamma + \frac{(128+512\sigma_{\eps,\delta}^2/m)k\log T}{\gamma}+mk,
\end{align*}
where the first step follows from splitting the regret from before to two sums, the second step follows since in the first sum $\Delta_a<\gamma$ and since there are only $T$ samples in total throughout all arms, and in the second sum we apply Equation~\eqref{eq:RaBound}, and the final step follows since there are $k$ arms in total and since the elements in the sum satisfy $\Delta_a\in [\gamma,1]$.

We balance the first two terms by defining $\gamma$ to be $\gamma = \sqrt{\frac{(128+512\sigma_{\eps,\delta}^2/m)k\log T}{T}}$, so the total regret is:
$R\leq 2\sqrt{(128+512\sigma_{\eps,\delta}^2/m)k\cdot T\log T}+mk $.
By a similar argument to Equation~\eqref{eq:expectedRegret} conditioning on whether or not $C$ occurred, we get that the expected regret $R(T)$ satisfies 
$$R(T)\leq 2\sqrt{(128+512\sigma_{\eps,\delta}^2/m)k\cdot T\log T}+mk+T\cdot 4T^{-2}=O\left(\sqrt{\left(1+\frac{\sigma_{\eps,\delta}^2}{m}\right)kT\log T}+mk\right).$$
\end{proof}

Recall that the private binary summation mechanism $M_{sum}$'s error distribution is sub-Gaussian with variance $\sigma_{\eps,\delta}^2=O\left(\frac{\log \frac{1}{\delta}}{\eps^2}\right)$. We fix a concrete batch size $m$ in \SDPAE, and apply Theorem~\ref{thm:batchAERegret} to get the following corollary.
%\jay{rewritten sentence. Fine?}
\begin{corollary}\label{cor:batchedAESDPRegret}
    \SDPAE with a batch size of $m=\ceil{\sigma_{\eps,\delta}}=\Theta\left(\frac{\log\frac{1}{\delta}}{\eps^2}\right)$ is $(\eps,\delta)$-SDP and has a distribution-dependent regret of $$O\left(\sum_{a\in [k]:\Delta_a>0}\left(\frac{\log T}{\Delta_a}+\frac{\Delta_a\log\frac{1}{\delta}}{\eps^2}\right)\right)=O\left(\left(\sum_{a\in [k]:\Delta_a>0}\frac{\log T}{\Delta_a}\right)+\frac{k\log\frac{1}{\delta}}{\eps^2}\right),$$ and a distribution-independent regret of $O\left(\sqrt{kT\log T}+\frac{k\log\frac{1}{\delta}}{\eps^2}\right)$.
\end{corollary}

\subsection{\VBSDPAE: Variable Batch Shuffle Differentially Private Arm Elimination}\label{sec:optimalAlg}
In this section we modify \SDPAE to give an $(\eps,\delta)$-SDP algorithm for the MAB problem with improved additional regret.
Recall that \SDPAE has a distribution-dependent regret of $O\left(\left(\sum_{a\in [k]:\Delta_a>0}\frac{\log T}{\Delta_a}\right)+\frac{k\log\frac{1}{\delta}}{\eps^2}\right)$, whereas the best known regret for $(\eps,\delta)$ centralized-DP is due to Tossou and Dimitrakakis~\cite{tossou2016algorithms}, and is $O\left(\left(\sum_{a\in [k]:\Delta_a>0} \frac{\log T}{\Delta_a}\right)+\frac{k}{\eps}\right)$. Since any SDP algorithm can be emulated by a centralized-DP mechanism, we can only wish to match the regret of Tossou and Dimitrakakis~\cite{tossou2016algorithms}. Hence, the natural question is: \enquote{Can we reduce the dependence in $\eps$ of the additive (second) regret term from $\frac{1}{\eps^2}$ to $\frac{1}{\eps}$?}

Before presenting our solution, we first indicate two intuitive approaches that fail. The first is to continue using batches of constant size $m=\Omega(1/\eps^2)$. This idea fails since intuitively in the worst case we can expect a first batch regret of $m=\Omega(1/\eps^2)$ which already surpasses $O(1/\eps)$. 
The second is to try to directly adapt the algorithm of Tossou and Dimitrakakis~\cite{tossou2016algorithms}, which can be interpreted as \AE with batches of size $\ceil{1/\eps}$. Unfortunately, this adaptation fails since it requires that the added noise to the empirical mean of each arm decreases with time, whereas in the shuffle model, we must ensure privacy with respect to each batch equally and independently, so the total noise cannot decrease with time.

Now for our improved algorithm, observe that using a large batch size  increases the regret a lot for arms which are very suboptimal. This is since rather than pulling these arms only a few  times until we detect that they are suboptimal, we commit ourselves to pulling them throughout a large batch. On the other hand, fixing the batch size to be small increases the overall estimation error, since every application of the SDP summation algorithm $M_{sum}$ introduces an error which is of the same magnitude  (specifically, the error is sub-Gaussian with variance $\sigma_{\eps,\delta}^2=O\left(\frac{\log \frac{1}{\delta}}{\eps^2}\right)$) %\jay{Added this reminder of $\sigma_{\eps,\delta}$} 
independently of the batch size. Hence, more executions of $M_{sum}$ translates to more noise due to privacy. We therefore extend \SDPAE to support variable size batches, which start small and gradually increase. Intuitively, the smaller batches initially, allow us to quickly eliminate very suboptimal arms with only a small number of pulls. We gradually increase the batch size to reduce the per-user error introduced by the private binary summation mechanism. 
%\jay{I think the last two sentences are redundant given what Uri rewrote above, right?}
%We found that batch sizes that grow exponentially as $2^{\#phase}$ is a good balance.
Specifically, we double the batch size after each phase.

% \uri{I edited this last paragraph. The old paragraph is commented out in the tex file.}

%Now for our improved algorithm, observe that using batches of constant size  increases the  regret the most for arms which are very suboptimal. This is since rather than pulling these arms only a few  times until we detect that they are suboptimal, we commit ourselves to a whole constant size batch. Hence, we extend \SDPAE to support variable size batches, which start small and gradually increase, and use $M_{sum}$ which supports any number of users (batch size) $m\in \naturals$. Intuitively, the smaller  batches initially, allow us to  eliminate very suboptimal arms with only a small number of pulls. We gradually increase the batch size to reduce the per-user error introduced by the private binary summation mechanism. We found that batch sizes that grow exponentially as $2^{\#phase}$ is a good balance.

% \newpage

\begin{algorithm}[H]\label{alg:variablebatchae}
\SetAlgoLined
\textbf{Input:} privacy parameters $\eps$ and $\delta$ and horizon $T$.\\
\textbf{Initialize:} $\hat{S}_a^0=0$, $N_a^0=0$ and $\hat{\mu}_{a}^0=0$ for every $a\in [k]$\;
Let $\sigma_{\eps,\delta}^2=O\left(\frac{\log \frac{1}{\delta}}{\eps^2}\right)$ be the sub-Gaussian variance of the error distribution of $M_{sum}$\;
Let $V=[k]$ denote the set of viable arms\;
 \For{phase $t\gets1,2,\ldots$}{
    Let $m^t\gets2^t$\;
    \For{arm $a\in V$}{
        \For{each new user $i\gets1$ \KwTo $m^t$}{
            User $i$ pulls the arm $a$ and observes reward $r_{a,i}^t$\;
            If total number of arm samples  is $T$, exit\;
        }
        \textbf{Communication:} Perform private binary summation 
        $Z_a^t \gets M_{sum}\left(\{r_{a,i}^t\}_{i=1}^{m^t}\right)$\;
        \textbf{Server update:} Update $\hat{S}_{a}^{t} \gets \hat{S}_{a}^{t-1} + Z_a^t $,
        $N_{a}^{t} \gets N_{a}^{t-1} + m^t$, and finally $\hat{\mu}_{a}^{t} \gets \hat{S}_{a}^{t}/N_{a}^{t}$\;
    }
  \textbf{Confidence bounds:} For each arm $a$, calculate $I_a^t\gets\left(\frac{2\sqrt{t}\sigma_{\eps,\delta}}{N_a^t}+\frac{1}{\sqrt{N_a^t}}\right)\cdot\sqrt{2\log T}$, and the upper and lower confidence bounds $UCB_a^t \gets \hat{\mu}_{a}^{t}+I_a^t$ and $LCB_a^t \gets \hat{\mu}_{a}^{t}-I_a^t$\;
   \textbf{Elimination:} remove all arms $a$ from $V$ such that $UCB_a^t < \max_{a'\in S}LCB_{a'}^t$\;
 }
 \caption{\VBSDPAE (Variable Batch Shuffle Differentially Private Arm Elimination)}
\end{algorithm}
\subsubsection{Algorithm outline}
Algorithm~\ref{alg:variablebatchae} above, which we call \textit{Variable Batch Shuffle Differentially Private Arm Elimination (\VBSDPAE)}, consolidates the previous presentation using a different batch size $m^t=2^t$ for each phase $t$.

\subsubsection{Analysis}
The privacy is trivial, since in each batch we use the $(\eps,\delta)$-SDP mechanism $M_{sum}$. We now focus on regret.

To give a regret bound on \VBSDPAE, we follow a similar proof to that of Theorem~\ref{thm:batchAERegret}.

\begin{theorem}\label{thm:variablebatchAERegret}
    The algorithm \VBSDPAE is $(\eps,\delta)$-SDP, and has a distribution-dependent regret of $O\left(\left(\sum_{a\in [k]:\Delta_a>0}\frac{\log T}{\Delta_a}\right)+k\sigma_{\eps,\delta}\log T\right)$, and a distribution-independent regret of $O\left(\sqrt{kT\log T}+k\sigma_{\eps,\delta}\log T\right)$.
\end{theorem}
\begin{proof}
We continue identically to the proof of Theorem~\ref{thm:batchAERegret}, except the fact that we need  the $t$'th index of each arm's private-binary-summation-error tape to contain an iid sample of the error of the private binary summation mechanism $M_{sum}$ for $m^t=2^t$ users. By the definition of $M_{sum}$, which is sub-Gaussian with the same variance for any number of users (batch size), the application of Hoeffding inequality as in the proof of Theorem~\ref{thm:batchAERegret} still follows. We conclude that the \textit{clean event} $C:=\left\{\forall a\in[k], \forall t\in[T]~~\abs{\hat{\mu}^t_a-\mu_a}\leq I_a^t\right\},$ where $I_a^t:=\left(\frac{2\sqrt{t}\sigma_{\eps,\delta}}{N_a^t}+\frac{1}{\sqrt{N_a^t}}\right)\cdot \sqrt{2\log T}$ occurs with high probability, that is $P\left(C\right)\geq1-4T^{-2}$.

For the regret analysis, we assume the clean event $C$. Let $a$ be a suboptimal arm, and let $t_0$ be the last phase following which we did not remove the arm $a$ yet (or the last phase if $a$ remains active to the end). As in the proof of Theorem~\ref{thm:batchAERegret}, we get that
\begin{align}
\Delta_a\leq \left(\frac{8\sqrt{t_0}\sigma_{\eps,\delta}}{N_a^{t_0}}+\frac{4}{\sqrt{N_a^{t_0}}}\right)\cdot \sqrt{2\log T}.\label{eq:variabledeltaaBound}
\end{align}

We now diverge from the proof of Theorem~\ref{thm:batchAERegret}. 
% First, recall that since $m^t=2^t$, so $N_a^t=\sum_{s=1}^{t}m^s=2^{t+1}-2\geq 2^t$ since $t\geq 1$, and specifically for any $c>0$ it holds that \begin{align}
% N_a^t>\max(2c\log c,4)\Rightarrow \frac{\sqrt{t}}{N_a^t}\leq \frac{t}{N_a^t}\leq \frac{\log N_a^t}{N_a^t}< \frac{\log (2c\log c)}{2c\log c}<\frac{\log (c^2)}{2c\log c}=\frac{1}{c},\label{eq:nabound}
% \end{align} 
% where the first step follows since $\sqrt{t}\leq t$, the second step follows since $N_a^t\geq 2^t$, the third step follows since $\frac{\log x}{x}$ is a strictly decreasing function for $x>3$, the fourth step follows since $2 \log x< x$ for any $x>0$ and since $\log x$ is an increasing function of $x$. 

Observe that if both $N_a^{t_0}> \frac{128\log T}{\Delta_a^2}$ and $N_a^{t_0}>\frac{16\sqrt{{t_0}}\sigma_{\eps,\delta}\cdot \sqrt{2\log T}}{\Delta_a}$, then we get a contradiction to Equation~\eqref{eq:variabledeltaaBound} since $\frac{4\cdot \sqrt{2\log T}}{\sqrt{N_a^{t_0}}}<\frac{\Delta_a}{2}$ and $\frac{8\sqrt{{t_0}}\sigma_{\eps,\delta}\cdot \sqrt{2\log T}}{N_a^{t_0}}< \frac{\Delta_a}{2}$ respectively. Hence, $N_a^{t_0}\leq \max\left( \frac{128\log T}{\Delta_a^2},\frac{16\sqrt{{t_0}}\sigma_{\eps,\delta}\cdot \sqrt{2\log T}}{\Delta_a}\right)$.
Therefore the total regret on arm $a$ is
\begin{align}
R_a&\leq 4\Delta_aN_a^{t_0}\leq \max\left( \frac{512\log T}{\Delta_a},64\sqrt{{t_0}}\sigma_{\eps,\delta}\cdot \sqrt{2\log T}\right)\nonumber\\
&\leq \frac{512\log T}{\Delta_a}+64\sqrt{2}\sigma_{\eps,\delta}\cdot \log T,\label{eq:variableRaBound}
\end{align}
where the first step follows since arm $a$ is eliminated following phase ${t_0}+1$ (or if $t_0$ is the last phase, then we finish and don't sample $a$ after it) with batch size $2^{t_0+1}=2\cdot 2^{t_0}\leq 3\cdot N_a^{t_0}$, and the third step follows since $m^{t_0}=2^{t_0}$ and ${t_0}\geq 1$, so $T\geq N_a^t=\sum_{s=1}^{{t_0}}m^s=2^{t_0+1}-2\geq 2^{t_0}$ and therefore ${t_0}\leq \log T$.

Similarly to the proof of Theorem~\ref{thm:batchAERegret} which uses Equation~\eqref{eq:RaBound} to get the distribution-dependent bound, here we use the analogous Equation~\eqref{eq:variableRaBound} to conclude that the distribution-dependent bound is 

\begin{align*}
    R(T)&=O\left(\sum_{a\in [k]:\Delta_a>0}\left(\frac{\log T}{\Delta_a}+\sigma_{\eps,\delta}\log T\right)\right)=O\left(\left(\sum_{a\in [k]:\Delta_a>0}\frac{\log T}{\Delta_a}\right)+k\sigma_{\eps,\delta}\log T\right).
\end{align*}

Now for the distribution-independent bound, similarly to the proof of Theorem~\ref{thm:batchAERegret}, assuming the clean event $C$, for any $\gamma>0$ it holds that the total regret
$$R\leq T\cdot \gamma + \frac{512k\log T}{\gamma}+64\sqrt{2}k\sigma_{\eps,\delta}\cdot \log T,$$
and specifically for $\gamma = \sqrt{\frac{512k\log T}{T}}$, the total regret
$R\leq \sqrt{2048k\cdot T\log T}+64\sqrt{2}k\sigma_{\eps,\delta}\cdot \log T $.
Similarly to the argument in Equation~\eqref{eq:expectedRegret}, conditioning on whether the clean event $C$ occurred or not, we conclude that the expected regret $R(T)$ satisfies
$$R(T)\leq \sqrt{2048k\cdot T\log T}+64\sqrt{2}k\sigma_{\eps,\delta}\cdot \log T+T\cdot 4T^{-2}=O\left(\sqrt{kT\log T}+k\sigma_{\eps,\delta}\log T\right).$$
\end{proof}

We recall that the private binary summation mechanism $M_{sum}$'s error distribution is sub-Gaussian with variance $\sigma_{\eps,\delta}^2=O\left(\frac{\log \frac{1}{\delta}}{\eps^2}\right)$, i.e., $\sigma_{\eps,\delta}=O\left(\frac{\sqrt{\log \frac{1}{\delta}}}{\eps}\right)$, and apply Theorem~\ref{thm:variablebatchAERegret} to get the following corollary.%\jay{rewritten sentence. Fine?}

\begin{corollary}\label{cor:variablebatchedAESDPRegret}
    \VBSDPAE is $(\eps,\delta)$-SDP and has a distribution-dependent regret of $O\left(\left(\sum_{a\in [k]:\Delta_a>0}\frac{\log T}{\Delta_a}\right)+\frac{k\sqrt{\log\frac{1}{\delta}}\log T}{\eps}\right)$, and a distribution-independent regret of $O\left(\sqrt{kT\log T}+\frac{k\sqrt{\log\frac{1}{\delta}}\log T}{\eps}\right)$.
\end{corollary}

\section{Conclusion and future work}
In this paper, we gave and analyzed differentially private algorithms for the MAB problem, closing the inevitable multiplicative $\Omega(1/\eps^2)$ regret gap between the local model and the centralized model, by considering the (intermediate) shuffle model. Our algorithms are batched variants of \AE, which use a private binary summation mechanism for the shuffle model as a building block. Compared to the non-private \AE algorithm's regret, our first algorithm \SDPAE has an additive factor of $\frac{k\log \frac{1}{\delta}}{\eps^2}$ using constant size batches, and our second algorithm \VBSDPAE improves the additive factor to $\frac{k\sqrt{\log \frac{1}{\delta}}\log T}{\eps}$ by using exponentially growing batches, which enable the early detection and elimination of very suboptimal arms.

A natural future work is to extend our results (i.e., the usage of a private binary summation mechanism for the shuffle model) to more general RL settings such as linear/contextual bandits or Markov decision processes. It would also be interesting to study whether our $\log T$ term in the additional additive regret factor can be shaved through a more sophisticated algorithm, or an alternative analysis.

\section*{Disclosure of Funding}
This work is partially supported by Israel Science Foundation (grants 993/17,1595/19,1871/19), German-Israeli Foundation (grant 1367/2017), Len Blavatnik and the Blavatnik Family Foundation, the European Research Council (ERC) under the European Union’s Horizon 2020 research and innovation program (grant agreement 882396),  the Yandex Initiative for Machine Learning at Tel Aviv University.

% \section*{Broader impact}
% This work is theoretical and we believe it can cause no direct negative societal impact or harm unless deployed in the industry. \haim{"negative" ? this does not sound right ?}\jay{better?}
% Real-world bandit algorithms which rely on the private user data are common, and are used in recommender systems, clinical trials and many more applications. We propose a differentially private mechanism which unlike the centralized model, protects this raw data from reaching the server, but unlike the local model, has regret which is similar to the non-private bandits problem.
% Since our algorithms are simple, easy to implement, and improve user privacy without substantial additional regret, we believe that they are applicable in practice, and can benefit user privacy in many applications. This work assumes that the agents are are cooperative and communicate only with the server, and the performance of our algorithms fails gracefully as more and more users act adversarially or uncooperatively.

\bibliographystyle{abbrv}
\newpage
\bibliography{references}

\begin{thebibliography}{10}

\bibitem{balle2019differentially}
B.~Balle, J.~Bell, A.~Gascon, and K.~Nissim.
\newblock Differentially private summation with multi-message shuffling.
\newblock {\em arXiv preprint arXiv:1906.09116}, 2019.

\bibitem{balle2019privacy}
B.~Balle, J.~Bell, A.~Gasc{\'o}n, and K.~Nissim.
\newblock The privacy blanket of the shuffle model.
\newblock In {\em Annual International Cryptology Conference}, pages 638--667.
  Springer, 2019.

\bibitem{balle2020private}
B.~Balle, J.~Bell, A.~Gasc{\'o}n, and K.~Nissim.
\newblock Private summation in the multi-message shuffle model.
\newblock In {\em Proceedings of the 2020 ACM SIGSAC Conference on Computer and
  Communications Security}, pages 657--676, 2020.

\bibitem{beimel2020round}
A.~Beimel, I.~Haitner, K.~Nissim, and U.~Stemmer.
\newblock On the round complexity of the shuffle model.
\newblock In {\em Theory of Cryptography Conference}, pages 683--712. Springer,
  2020.

\bibitem{beimel2008distributed}
A.~Beimel, K.~Nissim, and E.~Omri.
\newblock Distributed private data analysis: Simultaneously solving how and
  what.
\newblock In {\em Annual International Cryptology Conference}, pages 451--468.
  Springer, 2008.

\bibitem{bittau2017prochlo}
A.~Bittau, {\'U}.~Erlingsson, P.~Maniatis, I.~Mironov, A.~Raghunathan, D.~Lie,
  M.~Rudominer, U.~Kode, J.~Tinnes, and B.~Seefeld.
\newblock Prochlo: Strong privacy for analytics in the crowd.
\newblock In {\em Proceedings of the 26th Symposium on Operating Systems
  Principles}, pages 441--459, 2017.

\bibitem{bun2019heavy}
M.~Bun, J.~Nelson, and U.~Stemmer.
\newblock Heavy hitters and the structure of local privacy.
\newblock {\em ACM Transactions on Algorithms (TALG)}, 15(4):1--40, 2019.

\bibitem{chan2012optimal}
T.~H. Chan, E.~Shi, and D.~Song.
\newblock Optimal lower bound for differentially private multi-party
  aggregation.
\newblock In {\em European Symposium on Algorithms}, pages 277--288. Springer,
  2012.

\bibitem{cheu2019distributed}
A.~Cheu, A.~Smith, J.~Ullman, D.~Zeber, and M.~Zhilyaev.
\newblock Distributed differential privacy via shuffling.
\newblock In {\em Annual International Conference on the Theory and
  Applications of Cryptographic Techniques}, pages 375--403. Springer, 2019.

\bibitem{dubey2020differentially}
A.~Dubey and A.~Pentland.
\newblock Differentially-private federated linear bandits.
\newblock {\em arXiv preprint arXiv:2010.11425}, 2020.

\bibitem{dwork2006calibrating}
C.~Dwork, F.~McSherry, K.~Nissim, and A.~Smith.
\newblock Calibrating noise to sensitivity in private data analysis.
\newblock In {\em Theory of cryptography conference}, pages 265--284. Springer,
  2006.

\bibitem{erlingsson2020encode}
{\'U}.~Erlingsson, V.~Feldman, I.~Mironov, A.~Raghunathan, S.~Song, K.~Talwar,
  and A.~Thakurta.
\newblock Encode, shuffle, analyze privacy revisited: Formalizations and
  empirical evaluation.
\newblock {\em arXiv preprint arXiv:2001.03618}, 2020.

\bibitem{erlingsson2019amplification}
{\'U}.~Erlingsson, V.~Feldman, I.~Mironov, A.~Raghunathan, K.~Talwar, and
  A.~Thakurta.
\newblock Amplification by shuffling: From local to central differential
  privacy via anonymity.
\newblock In {\em Proceedings of the Thirtieth Annual ACM-SIAM Symposium on
  Discrete Algorithms}, pages 2468--2479. SIAM, 2019.

\bibitem{esfandiari2019regret}
H.~Esfandiari, A.~Karbasi, A.~Mehrabian, and V.~Mirrokni.
\newblock Regret bounds for batched bandits.
\newblock {\em arXiv preprint arXiv:1910.04959}, 2019.

\bibitem{gao2019batched}
Z.~Gao, Y.~Han, Z.~Ren, and Z.~Zhou.
\newblock Batched multi-armed bandits problem.
\newblock {\em arXiv preprint arXiv:1904.01763}, 2019.

\bibitem{ghazi2019scalable}
B.~Ghazi, R.~Pagh, and A.~Velingker.
\newblock Scalable and differentially private distributed aggregation in the
  shuffled model.
\newblock {\em arXiv preprint arXiv:1906.08320}, 2019.

\bibitem{guha2013nearly}
A.~Guha~Thakurta and A.~Smith.
\newblock (nearly) optimal algorithms for private online learning in
  full-information and bandit settings.
\newblock {\em Advances in Neural Information Processing Systems},
  26:2733--2741, 2013.

\bibitem{hoeffding1994probability}
W.~Hoeffding.
\newblock Probability inequalities for sums of bounded random variables.
\newblock In {\em The Collected Works of Wassily Hoeffding}, pages 409--426.
  Springer, 1994.

\bibitem{jain2012differentially}
P.~Jain, P.~Kothari, and A.~Thakurta.
\newblock Differentially private online learning.
\newblock In {\em Conference on Learning Theory}, pages 24--1. JMLR Workshop
  and Conference Proceedings, 2012.

\bibitem{mishra2015nearly}
N.~Mishra and A.~Thakurta.
\newblock (nearly) optimal differentially private stochastic multi-arm bandits.
\newblock In {\em Proceedings of the Thirty-First Conference on Uncertainty in
  Artificial Intelligence}, pages 592--601, 2015.

\bibitem{ren2020multi}
W.~Ren, X.~Zhou, J.~Liu, and N.~B. Shroff.
\newblock Multi-armed bandits with local differential privacy.
\newblock {\em arXiv preprint arXiv:2007.03121}, 2020.

\bibitem{robbins1952some}
H.~Robbins.
\newblock Some aspects of the sequential design of experiments.
\newblock {\em Bulletin of the American Mathematical Society}, 58(5):527--535,
  1952.

\bibitem{sajed2019optimal}
T.~Sajed and O.~Sheffet.
\newblock An optimal private stochastic-mab algorithm based on optimal private
  stopping rule.
\newblock In {\em International Conference on Machine Learning}, pages
  5579--5588. PMLR, 2019.

\bibitem{shariff2018differentially}
R.~Shariff and O.~Sheffet.
\newblock Differentially private contextual linear bandits.
\newblock {\em arXiv preprint arXiv:1810.00068}, 2018.

\bibitem{shi2021federated}
C.~Shi and C.~Shen.
\newblock Federated multi-armed bandits.
\newblock {\em arXiv preprint arXiv:2101.12204}, 2021.

\bibitem{thompson1933likelihood}
W.~R. Thompson.
\newblock On the likelihood that one unknown probability exceeds another in
  view of the evidence of two samples.
\newblock {\em Biometrika}, 25(3/4):285--294, 1933.

\bibitem{tossou2016algorithms}
A.~Tossou and C.~Dimitrakakis.
\newblock Algorithms for differentially private multi-armed bandits.
\newblock In {\em Proceedings of the AAAI Conference on Artificial
  Intelligence}, volume~30, 2016.

\bibitem{zheng2020locally}
K.~Zheng, T.~Cai, W.~Huang, Z.~Li, and L.~Wang.
\newblock Locally differentially private (contextual) bandits learning.
\newblock {\em arXiv preprint arXiv:2006.00701}, 2020.

\end{thebibliography}

\newpage
\textbf{\Large Appendix}
\appendix
% \begin{appendices}

% \section{Probability concentration bound}
% The following is a standard bound from Hoeffding~\cite{hoeffding1994probability},
% \begin{lemma}[Hoeffding's inequality]\label{lma:hoeffding}
%     Let $\{x_i\}_{i=1}^n$ be independent random variables, and let $\bar{X}=\frac{1}{n}\sum_i x_i$ and $S = \sum_i x_i$. Then,
%     \begin{itemize}
%         \item If $\forall i,~P\left(x_i\in [a,b]\right)=1$, then
%     $$P\left(\abs{\bar{X}-E[\bar{X}}\geq t\right)\leq 2\exp\left(\frac{-2nt^2}{(b-a)^2}\right),~~~~P\left(\abs{S-E[S]}\geq t\right)\leq 2\exp\left(\frac{-2t^2}{n(b-a)^2}\right).$$
%         \item If $\forall i,~P\left(x_i\in [a_i,b_i]\right)=1$, then
%     $$P\left(\abs{\bar{X}-E[\bar{X}}\geq t\right)\leq 2\exp\left(\frac{-2n^2t^2}{\sum_{i}(b_i-a_i)^2}\right),~~~~P\left(\abs{S-E[S]}\geq t\right)\leq 2\exp\left(\frac{-2t^2}{\sum_{i}(b_i-a_i)^2}\right).$$
%         \item If $\forall i,x_i\sim SG(\sigma)$, then
%     $$P\left(\abs{\bar{X}-E[\bar{X}}\geq t\right)\leq 2\exp\left(\frac{-nt^2}{2\sigma^2}\right),~~~~P\left(\abs{S-E[S]}\geq t\right)\leq 2\exp\left(\frac{-t^2}{2n\sigma^2}\right).$$
%         \item If $\forall i,x_i\sim SG(\sigma_i)$, then
%     $$P\left(\abs{\bar{X}-E[\bar{X}}\geq t\right)\leq 2\exp\left(\frac{-n^2t^2}{2\sum_{i}\sigma_i^2}\right),~~~~P\left(\abs{S-E[S]}\geq t\right)\leq 2\exp\left(\frac{-t^2}{2\sum_{i}\sigma_i^2}\right).$$
    
%     \end{itemize}
% \end{lemma}
% \section{Missing proofs from Section~\ref{sec:sdpMAB}}
% \begin{proof}[Proof of Theorem~\ref{thm:batchAERegret}]

% \end{proof}

% \section{Missing proofs from Section~\ref{sec:optimalAlg}}

\section{Private binary summation mechanism for the shuffle model}\label{sec:privateSummation}
In this section, for any $\eps,\delta\in (0,1)$ and number of users, we give an $(\eps,\delta)$-SDP private binary summation mechanism for the shuffle model, with an (additive) error distribution which is unbiased and sub-Gaussian with variance $\sigma_{\eps, \delta}^2 =O\left(\frac{\log (1/\delta)}{\eps^2}\right)$, and which does not depend on the input. Consider a group of $m$ users, each with a binary value $x_i\in \{0,1\}$, and the target is to calculate the sum $\sum_{i=1}^{m}x_i$. Our mechanism splits to two different internal mechanisms based on whether $m$ is \enquote{small} or \enquote{large}. Intuitively, to ensure that we add noise which is roughly $\frac{1}{\eps}$, when we have less than roughly $\frac{1}{\eps^2}$ users, each one adds several bits of noise, and when we have more than roughly $\frac{1}{\eps^2}$ users, each one adds a single bit of noise with some bias. This mechanism is summarized below:
% with a value $N$ to be determined later:
% Let $N=\ceil{\frac{1}{\eps^2\cdot m}}$.
% \begin{itemize}
%     \item If $m\leq \frac{32\log (2/\delta)}{\eps^2}$, then each user $i$ sends the shuffler their true bit, and random iid bits $y_{i,1},\ldots,y_{i,p}\sim Bernoulli(1/2)$. 
%     \item If $m >\frac{32\log (2/\delta)}{\eps^2}$, then each user $i$ sends the shuffler their true bit, and a random bit $y_{i}\sim Bernoulli\left(\frac{32\log (2/\delta)}{2m\eps^2}\right)$.
% \end{itemize}
% Let $B$ denote the sum of the additional bits. The server then sums up the values to get $M(X)$, and debiases $z=B+\sum_{j=1}^{n}x_j - N/2$.

\begin{algorithm}[H]\label{alg:sumMechanism}
\SetAlgoLined
$\tau\gets\frac{96\log (2/\delta)}{\eps^2}$\;
~\\
// Local Randomizer\\
\SetKwProg{Fn}{Function}{:}{}
\Fn{$E(x)$}{
    \eIf{$m\leq \tau$}{
        \KwRet $(x,y_{1},\ldots,y_{p})$ where $\{y_{j}\}_{j=1}^{p}$ are iid $y_{j}\sim Bernoulli(1/2)$, and $p=\ceil{\frac{\tau}{m}}$\;
    }{
        \KwRet $(x,y)$ where $y\sim Bernoulli\left(\frac{\tau}{2m}\right)$\;
    }
}
~\\

// Analyzer\\
\SetKwProg{Fn}{Function}{:}{}
\Fn{$A(z_1,\ldots,z_n)$}{
    \eIf{$m\leq \tau$}{
        \KwRet $\sum_{j=1}^{n}z_1 - \ceil{\frac{\tau}{m}}\cdot m/2$\;
    }{
        \KwRet $\sum_{j=1}^{n}z_1 - \tau/2$\;
    }
    % $N\gets n-m$\;
    % \KwRet $\sum_{j=1}^{n}z_1 - \tau/2$\;
}
\caption{$(\eps,\delta)$-SDP binary summation mechanism for $m$ users}
\end{algorithm}

\begin{theorem}
    For any $m\in \naturals$, $\eps<1$ and $\delta>0$, 
    % taking 
    % $N=\frac{32\log (2/\delta)}{\eps^2}$,
    % the above mechanism
    Algorithm~\ref{alg:sumMechanism} is $(\eps,\delta)$-SDP, unbiased, and has an error distribution which is sub-Gaussian with variance $\sigma_{\eps, \delta}^2 =O\left(\frac{\log (1/\delta)}{\eps^2}\right)$ and independent of the input.
\end{theorem}

\begin{proof}
We first prove that the mechanism is $(\eps,\delta)$-SDP, and then prove the other claims. 

Indeed, consider two neighboring inputs $X=(0,x_2,\ldots,x_m)$ and $X'=(1,x_2,\ldots,x_m)$.
To ease on the analysis, we define the random variable $B$ to be the sum of all the random bits (i.e., $y$ or $y_1,\ldots,y_p$ depending on $m$) over all users in $X$. We define $B'$ identically with respect to $X'$.

We first claim that the sum $M^*(X) = B+\sum_{j=1}^{m}x_j$ of the shuffled reported bits is $(\eps,\delta)$-DP. 

Since $B$ is binomial in both regimes, by Chernoff bounds as in Theorem E.1 in Cheu et al.~\cite{cheu2019distributed}, for any $\delta>0$ it holds that $P\left(\abs{B-\expect{}{B}}\geq \sqrt{3\expect{}{B}\log \frac{2}{\delta}}\right)<\delta$. Therefore, define $I_c=\left(\expect{}{B}-\sqrt{3\expect{}{B}\log \frac{2}{\delta}}, \expect{}{B}+\sqrt{3\expect{}{B}\log \frac{2}{\delta}}\right)$, and we get that $P(B\notin I_c)\leq\delta$ (and similarly for $B'$).

To show that $ \frac{P(B=t)}{P(B'=t-1)}\leq e^\eps$ for any $t\in I_c$, we split to the two regimes of $m$: the small $m\leq \tau$ regime, and the large $m>\tau$ regime.

\paragraph{Small $m\leq \tau$:}
In this case, $B\sim Binomial(\ceil{\frac{\tau}{m}}\cdot m,1/2)$, so $\expect{}{B}=\ceil{\frac{\tau}{m}}\cdot m/2$. For any $t\in I_c$, it holds that

% \begin{align}
%     \frac{P(B=t)}{P(B'=t-1)}&=\frac{N-t+1}{t}\geq \frac{N-\left(\frac{N}{2}+c\right)+1}{\frac{N}{2}+c}\geq\frac{\frac{N}{2}-c}{\frac{N}{2}+c}\nonumber\\
%     &=\frac{N/2-\sqrt{(N/2)\log(2/\delta)}}{N/2+\sqrt{(N/2)\log(2/\delta)}}=\frac{\frac{32\log (4/\delta)}{\eps^2}(1-\eps/4)}{\frac{32\log (2/\delta)}{\eps^2}(1+\eps/4)}=\frac{1-\eps/4}{1+\eps/4}\geq e^{-\eps},\label{eq:BratioUpperBound}
% \end{align}
% where the first step follows since $B,B'$ are iid binomial, the second step follows since $t\in I_c\Rightarrow t\leq \frac{N}{2}+c$ and since $\frac{N+1-x}{x}$ is a decreasing function of $x$ in $x\in [0,N+1]$, the fifth step follows since $N=\frac{32\log (4/\delta)}{\eps^2}$, and the last step follows since $\frac{1-x/4}{1+x/4}\geq e^{-x}$ for any $x\in [0,1]$.

% For the other side,
\begin{align}
     \frac{P(B=t)}{P(B'=t-1)}& =\frac{2\expect{}{B}-t+1}{t}
     \leq \frac{\expect{}{B}+\sqrt{3\expect{}{B}\log \frac{2}{\delta}}+1}{\expect{}{B}-\sqrt{3\expect{}{B}\log \frac{2}{\delta}}}\nonumber\\
    &\leq \frac{\tau/2+\sqrt{\tau/2\cdot 3\log \frac{2}{\delta}}+1}{\tau/2-\sqrt{\tau/2\cdot 3\log \frac{2}{\delta}}}=\frac{1+\sqrt{6\log \frac{2}{\delta}/\tau}+2/\tau}{1-\sqrt{6\log \frac{2}{\delta}/\tau}}\nonumber\\
    &=\frac{1+\eps/4+2/\tau}{1-\eps/4}
    \leq \frac{1+\eps/4+\frac{\eps}{4}}{1-\eps/4}
    =\frac{1+\eps/2}{1-\eps/4}\leq e^{\eps},
\end{align}
where the first step follows since $B,B'$ are iid binomial with $\ceil{\frac{\tau}{m}}\cdot m=2\expect{}{B}$ trials of success probability $1/2$, the second step follows since $t\in I_c\Rightarrow t\geq \expect{}{B}-\sqrt{3\expect{}{B}\log \frac{2}{\delta}}$ (which is non-negative) and since $\frac{2\expect{}{B}+1-t}{t}$ is a decreasing function of $t$ for $t\geq 0$, the third step follows since $\frac{x+\sqrt{ax}+1}{x-\sqrt{ax}}$ is a decreasing function of $x$ for $x>a$, where we take $a=3\log \frac{2}{\delta}$ and $x=\expect{}{B}\geq \tau/2>a$, the fourth step follows by dividing the nominator and the denominator by $\tau/2$, the fifth step follows by the definition of $\tau$, the sixth step follows since $\eps<1$ so $\tau\geq 8/\eps$, and the last step follows since $\frac{1+x/2}{1-x/4}\leq e^{x}$ for any $x\in [0,1]$.
This concludes the case $m\leq\tau$.

\paragraph{Large $m>\tau$:}
In this case, $B\sim Binomial(m,\frac{\tau}{2m})$, so $\expect{}{B}=\tau/2$.
For any $t\in I_c$, it holds that

\begin{align}
     \frac{P(B=t)}{P(B'=t-1)}& =\frac{m-t+1}{t}\cdot \frac{\frac{\tau}{2m}}{1-\frac{\tau}{2m}}
     \leq \frac{m-\tau/2+\sqrt{\frac{3}{2}\tau\log \frac{2}{\delta}}+1}{\tau/2-\sqrt{\frac{3}{2}\tau\log \frac{2}{\delta}}}\cdot \frac{\frac{\tau}{2m}}{1-\frac{\tau}{2m}}\nonumber\\
    &=\frac{m-\tau/2+\sqrt{\frac{3}{2}\tau\log \frac{2}{\delta}}+1}{\tau/2-\sqrt{\frac{3}{2}\tau\log \frac{2}{\delta}}}\cdot \frac{\tau/2}{m-\tau/2}\nonumber\\
    &=\frac{m-\tau/2+\sqrt{\frac{3}{2}\tau\log \frac{2}{\delta}}+1}{m-\tau/2}\cdot \frac{\tau/2}{\tau/2-\sqrt{\frac{3}{2}\tau\log \frac{2}{\delta}}}\nonumber\\
    &= \left(1+\frac{\sqrt{\frac{3}{2}\tau\log \frac{2}{\delta}}+1}{m-\tau/2}\right)\cdot\frac{1}{1-\sqrt{6\log \frac{2}{\delta}/\tau}}\nonumber\\
    &\leq \left(1+\sqrt{6\log \frac{2}{\delta}/\tau}+2/\tau\right)\cdot\frac{1}{1-\sqrt{6\log \frac{2}{\delta}/\tau}}\nonumber\\
    &=\frac{1+\eps/4+2/\tau}{1-\eps/4}\leq \frac{1+\eps/4+\eps/4}{1-\eps/4}\leq e^{\eps},
\end{align}
% \tau/2-\sqrt{\frac{3}{2}\tau\log \frac{2}{\delta}}
where the first step follows since $B,B'$ are iid binomial with $m$ trials of success probability $\frac{\tau}{2m}$, the second step follows since $t\in I_c\Rightarrow t\geq \expect{}{B}-\sqrt{3\expect{}{B}\log \frac{2}{\delta}}=\tau/2-\sqrt{\frac{3}{2}\tau\log \frac{2}{\delta}}$ (which is non-negative) and since $\frac{m+1-t}{t}$ is a decreasing function of $t$ for $t\geq 0$, the sixth step follows since $m-\tau/2\geq \tau-\tau/2 =\tau/2$, the seventh step follows by the definition of $\tau$, the eighth step follows since $\eps<1$ so  $\tau\geq 8/\eps^2\geq 8/\eps$, and the last step follows since $\frac{1+x/2}{1-x/4}\leq e^{x}$ for any $x\in [0,1]$. This concludes the case $m>\tau$.

We therefore conclude that in both regimes of $m$, 
\begin{align}
    \forall t\in I_c,~ \frac{P(B=t)}{P(B'=t-1)}\leq e^\eps. \label{eq:BratioLowerBound}
\end{align}
A dual argument shows that $ \frac{P(B=t)}{P(B'=t-1)}\geq e^{-\eps}$ using the fact that $t\in I_c$ so $t\leq \expect{}{B}+\sqrt{3\expect{}{B}\log \frac{2}{\delta}}$ and substituting in the value of $\expect{}{B}$ as in the cases above. 

We define $k=\sum_{j=2}^{m}x_j$ to be the true sum of the bits of  $X$, and the true sum of the bits of $X'$ minus one.
Therefore, for any $F\subseteq \naturals$ it holds that 

\begin{align*}
    P(M^*(X)\in F)&=P(M^*(X)\in F\wedge  B\in I_c)+P(M^*(X)\in F\wedge  B\notin I_c)\\
    &\leq P(M^*(X)\in F\wedge  B\in I_c)+P(B\notin I_c)\\
    &\leq P(M^*(X)\in F\wedge  B\in I_c)+\delta\\
    &=\delta + \sum_{s\in F}P(M^*(X)=s\wedge  B\in I_c)\\
    &=\delta + \sum_{s\in F}P(B=s-k\wedge  B\in I_c)\\
    &=\delta + \sum_{s\in F}P(B=s-k\wedge  s-k\in I_c)\\
    &\leq \delta + \sum_{s\in F}e^{\eps}\cdot P(B'=s-k-1\wedge  s-k\in I_c)\\
    &=\delta + e^{\eps}\cdot \sum_{s\in F}P(M^*(X')=s\wedge  s-k\in I_c)\\
    &\leq \delta + e^{\eps}\cdot \sum_{s\in F}P(M^*(X')=s)\\
    &=\delta + e^{\eps}\cdot P(M^*(X')\in F),
\end{align*}
where the first step follows by the law of total probability, the third step follows since $P(B\notin I_c)\leq \delta$, the fourth step follows by the law of total probability, the fifth step follows by the definition of $M^*(X)=k+B$, the seventh step follows by substituting $t\gets s-k\in I_c$ into Equation~\eqref{eq:BratioLowerBound}, and the eighth step follows by the definition of $M^*(X')=k+1+B'$.
A similar dual argument uses the fact that $\forall t\in I_c,~ \frac{P(B=t)}{P(B'=t-1)}\geq e^{-\eps}$ to show that $P(M^*(X')\in F) \leq \delta + e^{\eps}P(M^*(X)\in F)$, and we conclude that $M^*$ is $(\eps,\delta)$-DP.

To see that $M$ is $(\eps,\delta)$-SDP, note that in our mechanism $M(X)$, given the number of users $m$, the total number of bits $U$ that the server receives is constant. Therefore, the shuffler's output is a random permutation of its input, which is of constant size. Thus, the shuffler's output's distribution is identical to the output distribution of the mechanism which first selects the number $s$ of ones in the shuffler's input where $s\sim M^*(X)$, and then post-processes the output $s$ by outputting a randomly shuffled binary vector with $s$ ones, and $U-s$ zeros. Since we have shown that $M^*$ is $(\eps,\delta)$-DP, by post-processing arguments we conclude the shuffler's output is $(\eps,\delta)$-DP, so $M$ is $(\eps,\delta)$-SDP.

Now for the other claims, first recall that in both regimes of $m$, the output of the mechanism is of the form $z=B+\sum_{j=1}^{m}x_j - \expect{}{B}$ where $B$ is the only source of randomness in the mechanism. Therefore, the mechanism is unbiased since $\expect{}{z-\sum_{j=1}^{m}x_j}=\expect{}{B- \expect{}{B}}=0$.  In addition, the (additive) error of the mechanism which is precisely $B- \expect{}{B}$, is obviously independent of the input $\{x_i\}_{i=1}^{m}$, and only depends on the natural parameters of the problem.

Finally, to see that the mechanism's additive error is sub-Gaussian with variance $O\left(\frac{\log (1/\delta)}{\eps^2}\right)$, it suffices to show that $B$ is sub-Gaussian with variance $O\left(\frac{\log (1/\delta)}{\eps^2}\right)$, since $B$ is the additive error shifted by a constant (this constant is $\expect{}{B}$ and adding constants does not change the sub-Gaussian variance). 
Indeed, recall that in both cases $B$ is binomial, where in the small $m\leq \tau$ case $\expect{}{B}=\ceil{\frac{\tau}{m}}\cdot m/2\leq (\tau +m)/2\leq (\tau +\tau)/2 = \tau$, and in the large $m> \tau$ case $\expect{}{B}=\tau/2\leq \tau$ as well. Therefore, by Chernoff bounds as in Theorem E.1 in Cheu et al.~\cite{cheu2019distributed}, we get that for any $t>0$, $P\left(B-\expect{}{B}\leq t\right)<\exp\left(\frac{-t^2}{3\expect{}{B}}\right)\leq \exp\left(\frac{-t^2}{3\tau}\right) $ and $P\left(B-\expect{}{B}\geq -t\right)<\exp\left(\frac{-t^2}{3\expect{}{B}}\right)=\exp\left(\frac{-t^2}{3\tau}\right)$, so by the equivalent definition of a sub-Gaussian variable, $B$ is sub-Gaussian with parameter $O(\tau)=O\left(\frac{\log (1/\delta)}{\eps^2}\right)$.
\end{proof}

\end{document}